\documentclass[10pt,twocolumn,letterpaper]{article}

\usepackage[pagenumbers]{cvpr} 

\usepackage{bbm}
\renewcommand{\mathbb}{\mathbbm}
\usepackage{todonotes,soul}
\usepackage{amsthm}
\usepackage{xcolor}
\usepackage{multirow,subcaption}  
\newtheorem{proposition}{Proposition}
\usepackage{csquotes}\MakeOuterQuote{"}
\renewcommand{\paragraph}[1]{\par\noindent\textbf{#1}}
\usepackage{comment}\excludecomment{exclude}
\usepackage{soul}

\definecolor{cvprblue}{rgb}{0.21,0.49,0.74}
\usepackage[pagebackref,breaklinks,colorlinks,allcolors=cvprblue]{hyperref}

\def\paperID{*****} 
\def\confName{CVPR}
\def\confYear{2026}

\title{Bag of Bags: Adaptive Visual Vocabularies 
for Genizah Join Image Retrieval}

\author{Sharva Gogawale, Gal Grudka, Daria Vasyutinsky-Shapira,\\ Omer Ventura, Berat Kurar-Barakat, Nachum Dershowitz\\
School of Computer Science and AI, Tel Aviv University, Ramat Aviv, Israel\\
{\tt\small\{sharvag,galgrudka,omerventura\}@mail.tau.ac.il}, 
{\tt\small\{dariashap,berat,nachum\}@tau.ac.il}
}

\begin{document}
\maketitle

\begin{abstract}
A join is a set of manuscript fragments identified as originally emanating from the same manuscript. We study manuscript join retrieval: Given a query image of a fragment, retrieve other fragments originating from the same physical manuscript. 
We propose Bag of Bags (BoB), an image-level representation that replaces the global-level visual codebook of classical Bag of Words (BoW) with a fragment-specific vocabulary of local visual words. 
Our pipeline trains a sparse convolutional autoencoder on binarized fragment patches, encodes connected components from each page, clusters the resulting embeddings with per-image $k$-means, and compares images using set-to-set distances between their local vocabularies. 
Evaluated on fragments from the Cairo Genizah, the best BoB variant (viz.\@ Chamfer) achieves Hit@1 of 0.78 and MRR of 0.84, compared to 0.74 and 0.80, respectively, for the strongest BoW baseline (BoW-RawPatches-$\chi^2$), a 6.1\% relative improvement in top-1 accuracy. 
We furthermore study a mass-weighted BoB-OT variant that incorporates cluster population into prototype matching and present a formal approximation guarantee bounding its deviation from full component-level optimal transport. 
A two-stage pipeline using a BoW shortlist followed by BoB-OT reranking provides a practical compromise between retrieval strength and computational cost, supporting applicability to larger manuscript collections. The code and dataset are available at \url{https://github.com/TAU-CH/midrash_bob}.

\end{abstract}

\section{Introduction}
\label{sec:intro}

The Cairo Genizah is a unique source of preserved fragmented medieval manuscripts, accumulated between the 11th and  19th centuries in the Ben Ezra Synagogue in Old Cairo and now dispersed across dozens of libraries and private collections worldwide. 
The manuscripts are mostly written in Hebrew, Aramaic, and Judeo-Arabic. 
Their study over the past century and a quarter has had enormous impact on our knowledge of medieval Mediterranean history, literature, commerce, and culture.

To facilitate more thorough study of the Genizah texts, there is a strong need to group related fragments and reconstruct as much as possible of their original manuscripts. 
Over decades, scholars have spent considerable effort manually identifying such groups, known as \textit{joins}: manuscript fragments that originally belonged to the same physical codex or scroll, but have since been separated, requiring researchers to examine them in person across institutions. Manual identification remains the gold standard for discovering joins; however, it does not scale well and cannot be applied to the entire collection. 

Automating join retrieval is a challenging computer vision task because many manuscript fragments are often severely damaged, incomplete, and stained. 
The difficulty arises from the fact that fragments originating from the same manuscript may share only subtle, handwriting-style-specific visual patterns, while fragments from different manuscripts can still appear globally similar in aspects such as background texture, aging, and   degradation level.

Standard retrieval architectures frequently rely on Bag of Words (BoW) paradigms. BoW maps images to histograms over a shared global codebook. It summarizes all fragments with respect to the same dictionary. However, this global quantization is fundamentally limited for paleographic analysis, as it causes critical image-specific handwriting style information to be lost during quantization. Two images with identical codeword frequencies but different geometric character style distributions are assigned distance zero by BoW, yet they may have been written by entirely different hands. Conversely, fragments from the same manuscript may differ in lighting and layout, causing BoW to miss them. BoW's pooling over a global codebook erases the image-specific handwriting style signal.


Key contributions are as follows:
\begin{itemize}
\item We propose Bag of Bags (BoB), a fragment-specific representation for fragment retrieval that replaces a shared global visual-word codebook with per-fragment vocabularies of local visual words, constructed from sparse autoencoder embeddings of connected-component 
patches. 

\item We instantiate BoB with three set-to-set distances---bipartite 
assignment, symmetric Chamfer, and mass-weighted optimal transport (OT)---and show empirically that soft nearest-neighbor matching (Chamfer) is 
most robust for partially damaged fragments, while weighted OT provides 
a principled mass-aware alternative with formal approximation guarantees.

\item We present a two-stage retrieval pipeline in which BoW-Cosine 
generates candidates and BoB-OT reranks only the shortlist, reducing 
online cost to $\mathcal{O}(M \cdot K^3)$ independent of gallery size 
while closely preserving the accuracy of exhaustive BoB-OT. 

\item We provide a detailed ablation on vocabulary size, encoder 
dimension, activation sparsity, and component normalization, showing 
that BoB's page-adaptive vocabulary structure contributes retrieval 
gains beyond what the encoder quality alone explains.
\end{itemize}

\section{Related Work}
\label{sec:rl}
The Cairo Genizah, stored in the  Ben Ezra Synagogue in Fustat (Old Cairo) and discovered and retrieved at the end of the 19th century, is a large and historically significant collection of manuscript fragments dating mainly from the 10th to the 15th centuries. 
These documents were subsequently dispersed across more than fifty libraries and collections worldwide~\cite{penn}.  
A fundamental challenge facing Genizah scholarship is that many leaves were discovered as loose or damaged fragments. 
Identifying joins
has long been a central scholarly task.  The results of these manual efforts can be found on the site of the \href{https://fjms.genizah.org}{Friedberg Genizah Project}, currently in transition to the \href{https://www.nli.org.il/en/discover/manuscripts/hebrew-manuscripts}{National Library of Israel}. 
Though thousands of joins have already been documented, substantial manual effort is still required
\cite{penn}. 
This motivates automated retrieval systems that, given a query fragment, rank likely join candidates for expert inspection.  

Automated identification of joins was first explored in~\cite{wolf}, who proposed a method for finding candidate matches between manuscript fragments. Their approach represents each image using local visual features and a bag of features representation, and then compares pairs of leaves using a learned similarity measure to determine whether they may belong to the same original work. However, this method requires computationally expensive pairwise comparisons during retrieval. 
Analogous fragment assembly problems arise in other historical document collections, including ancient papyri, where AI-based approaches have been surveyed for matching dispersed
fragments via visual texture and learned similarity measures~\cite{pap,Isabelle}.

More broadly, historical document retrieval has been studied through word and pattern spotting, where the goal is to retrieve repeated words or writer-specific patterns without full OCR. 
An efficient learning-free word spotting method based on local binary patterns was proposed for handwritten historical documents, demonstrating robustness to degradation and variations across multiple writers~\cite{lbp}.
Bag of Visual Words (BoVW) and Bag of Words (BoW) models constitute a fundamental paradigm in image retrieval~\cite{vg}. These approaches quantize local image descriptors into a shared codebook, representing each image as a histogram of visual words. 
Extending this to documents, Shekhar showed that recognition-free retrieval is possible by quantizing local descriptors and comparing fixed-dimensional histograms~\cite{shekhar}.
A two-stage system for word and pattern spotting in historical manuscripts combines BoVW for candidate generation with a spatial verification step based on the Longest Weighted Profile algorithm, enabling retrieval from a single query~\cite{spotit}. These methods are closely related, but they target repeated words or local patterns rather than page-level manuscript identity. This global quantization can be challenging for fine-grained scribal analysis: fragments with identical global word counts but entirely different geometric distributions in feature space receive identical representations. Optimal transport offers a more flexible alternative for comparing visual distributions. Earth Mover’s Distance (EMD) utilizes optimal transport to compare feature distributions~\cite{earth}, recently extended to match continuous GMMs~\cite{gmm}. 
We combine these threads: instead of assigning all pages to one global codebook, we learn a fragment-specific vocabulary and compare fragments through set-to-set matching. 

\begin{figure}[t]
  \centering
    \includegraphics[width=\linewidth]{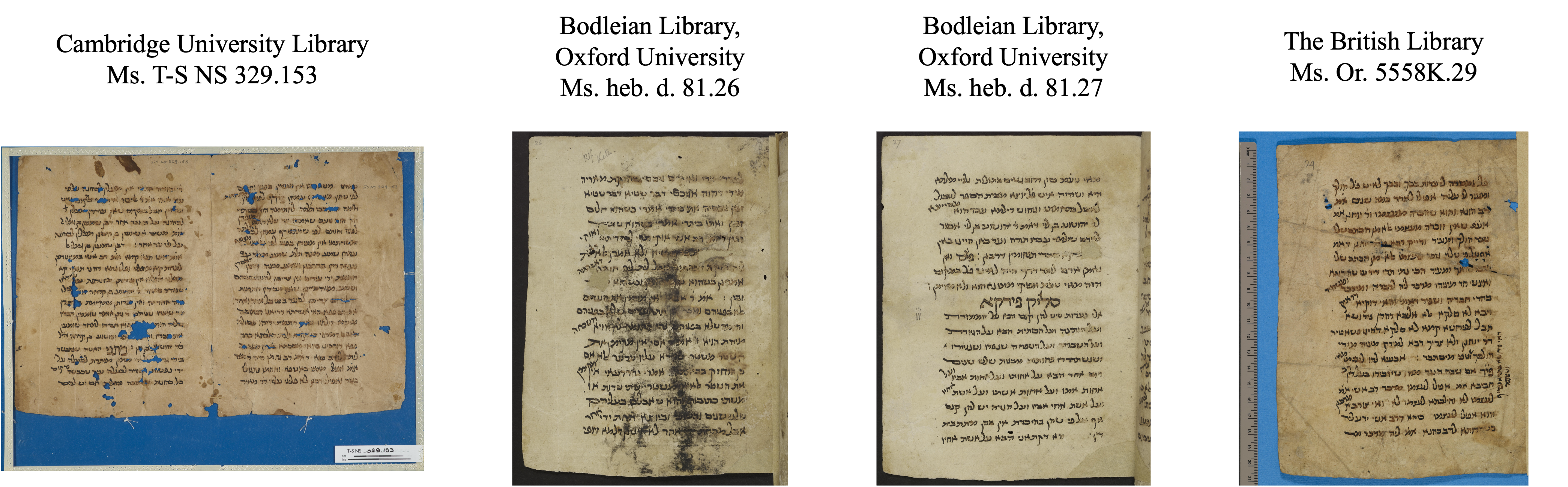}
    \includegraphics[width=\linewidth]{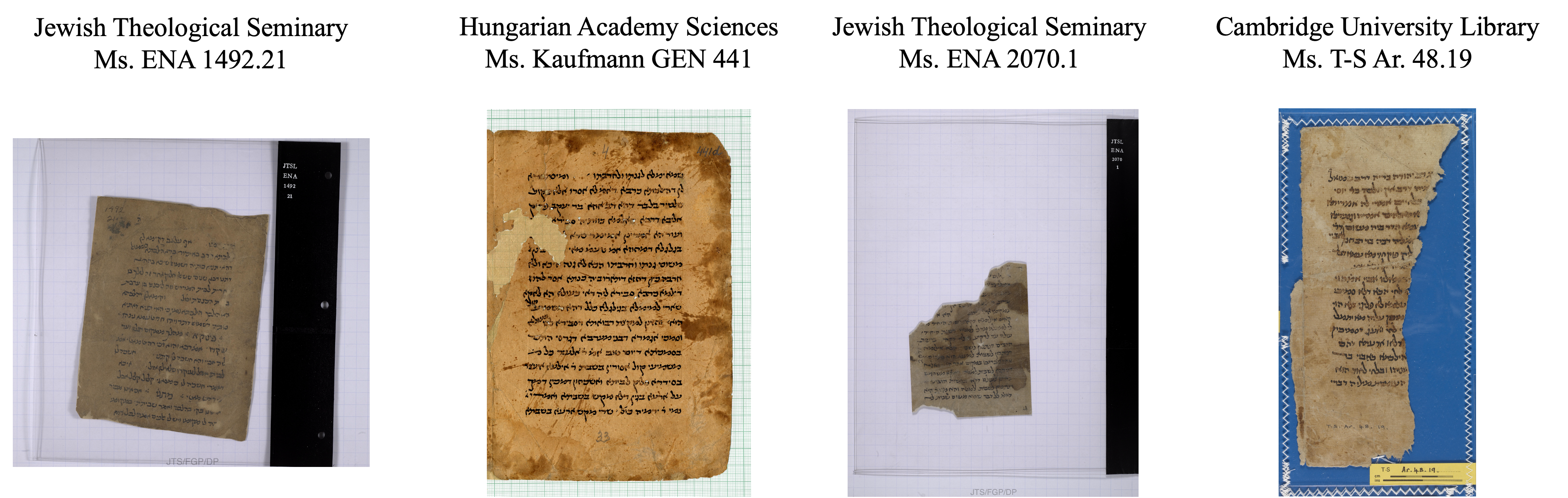}
  
  \includegraphics[width=\linewidth]{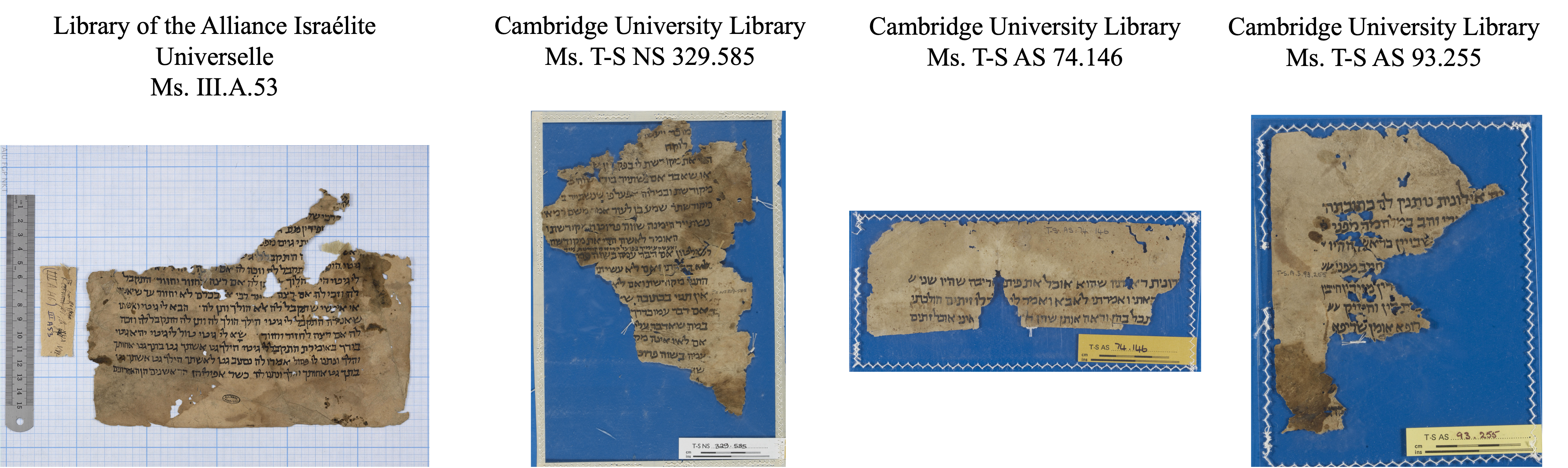}
  
  \caption{Representative join groups from the Cairo Genizah benchmark.
  Each row shows fragments originating from the same physical manuscript,
  held among different institutions worldwide.
  (The blue backgrounds were designed for easy algorithmic segmentation.)}
  \label{fig:all}
  
\end{figure}

A closely related research area is writer identification~\cite{dargan2019writer}, which seeks to determine the identity of the writer of a specific document based on handwriting out of a given set of writers for whom writing samples are supplied.  
Computer vision tools have been developed to assist paleographers directly in scribal hand identification, producing interpretable patch-level similarity maps that
experts can interrogate alongside traditional paleographic
evidence~\cite{grieggs}.
Existing methods are commonly divided into texture-based and grapheme-based approaches. 
Texture-based methods~\cite{bulacu2003writer, bulacu2007automatic, brink2012writer} characterize handwriting using statistical properties of the written trace, such as slant, curvature, and texture, whereas grapheme-based methods extract local structures and map them into a shared feature space. 
The latter can be formed using BoW, which relies on zero-order statistics by counting assignments of local descriptors to the nearest visual words in a predefined vocabulary~\cite{fiel2012writer, xiong2015text} or by richer aggregation schemes such as Fisher vectors, which encode first and second-order deviations of local descriptors from a probabilistic visual vocabulary~\cite{fiel2013writer, jain2014combining}, and VLAD, which aggregates residuals between local descriptors and their nearest visual words~\cite{christlein2015writer, christlein2017unsupervised}. 


\section{Methodology}
\label{sec:method}

\subsection{Problem Formulation}
\label{subsec:problem}

Let $\mathcal{I} = \{I_i\}_{i=1}^N$ be a collection of manuscript fragment images partitioned into join clusters
$\mathcal{C} = \{C_j\}_{j=1}^M$, where $C_j \subseteq \mathcal{I}$ denotes a set of fragment images originating from the same physical manuscript.
Given a query fragment $I_q$, the \emph{retrieval task} is to rank all other images so that members of $C_{j(q)}$, where $j(q)$ denotes the cluster index of $I_q$, appear at top ranks. We evaluate retrieval using Hit@$k$, mAP@$k$, mean reciprocal rank (MRR), and Macro-F1@1.
The implemented system constructs two competing representations in a shared latent space: (i) an image-adaptive \emph{Bag of Bags} (BoB) representation, and (ii) a global codebook Bag of Words (BoW) baseline.
The BoB pipeline proceeds in three stages:
(1)~character-anchored patch extraction via connected component
analysis; (2)~sparse autoencoder encoding of each patch; and
(3)~per-image $k$-means clustering to produce a local prototype
vocabulary.  Fragment similarity is then measured as a set-to-set
distance between per-image vocabularies.

\subsection{Dataset}
\label{subsec:dataset}
Our dataset consists of 287 manuscript fragment images from the Cairo Genizah join benchmark, drawn from multiple institutions: Cambridge University Library (Taylor-Schechter and related collections), the Jewish Theological Seminary (JTS/ENA) in New York, the British Library, and the Alliance Israélite Universelle in Paris. See the samples in \cref{fig:all}. We validate consistency against ground-truth labels across 100 join clusters. 
The dataset is highly imbalanced, with cluster sizes ranging from 2--9 fragments. 
To account for this imbalance, in evaluation, we include Macro-F1@1 alongside Hit@1 and MRR; it averages the per-cluster F1 score, weighting each cluster equally regardless of size. 
The autoencoder is trained on individual patches without join-cluster supervision and is therefore unaffected by cluster-size skew.
The manuscripts are, for the most part, written in Hebrew letters across a variety of languages and script modes, some more square (with mainly disconnected letters) and others more cursive (with many interconnected letters). All experiments use binarized (Otsu-thresholded) grayscale images.
Our experiments are conducted on this manually annotated benchmark subset (287 images, 100 join clusters), drawn from a much larger Genizah collection containing approximately 250--300K fragment images overall. 
We focus on this annotated subset because reliable join labels are currently available at this scale, while the proposed retrieval pipeline is designed to support deployment over the full corpus.

\subsection{Character-Anchored Patch Extraction}
\label{sec:stage1}

For each image, we first binarize. On the resulting image $I_i$, we run 8-connected component analysis
to obtain the component set
$\mathcal{C}_i$. Images whose mean pixel intensity exceeds 128 are inverted so that text is always white (non-zero) for consistent
detection.

\paragraph{Area filtering.}
Given an input manuscript fragment $I$, we isolate individual character instances using connected component analysis. We filter the resulting component set $\mathcal{C}_i$ to exclude noise, background artifacts, and massive text blocks based on pixel area, yielding valid components 
\[
  \widetilde{\mathcal{C}}_i =
  \{\,c \in \mathcal{C}_i : A_{\min} \le \mathrm{area}(c) \le A_{\max}\,\}.
\] 
Thresholds $A_{\min} = 300$ and $A_{\max} = 3000$ were chosen based on visual inspection of the component size distribution: Components below 300 pixels correspond primarily to noise, ink specks, and binarization artifacts, while those above 3000 pixels typically correspond to large ink blobs or partially merged letter clusters that do not represent individual letterforms.

\paragraph{Fit-and-pad normalization.}
Each retained component is scaled so that its longest side equals 60 pixels
while preserving the aspect ratio and is then centered in a zero-padded
64$\times$64 patch. This yields scale-normalized patches while preserving
character geometry. The resulting patch set for image $I_i$ is
$\mathcal{P}_i = \{p_1^i, \ldots, p_{n_i}^i\}$.

\paragraph{Quality filtering.}
Each component bounding-box crop must contain at least 5\% white pixels.
After normalization, the final patch must also contain at least 2\% white
pixels. Pages yielding fewer than 200 valid components after all filters
are excluded.

\begin{figure*}[t]
\centering
\includegraphics[width=0.80\linewidth]{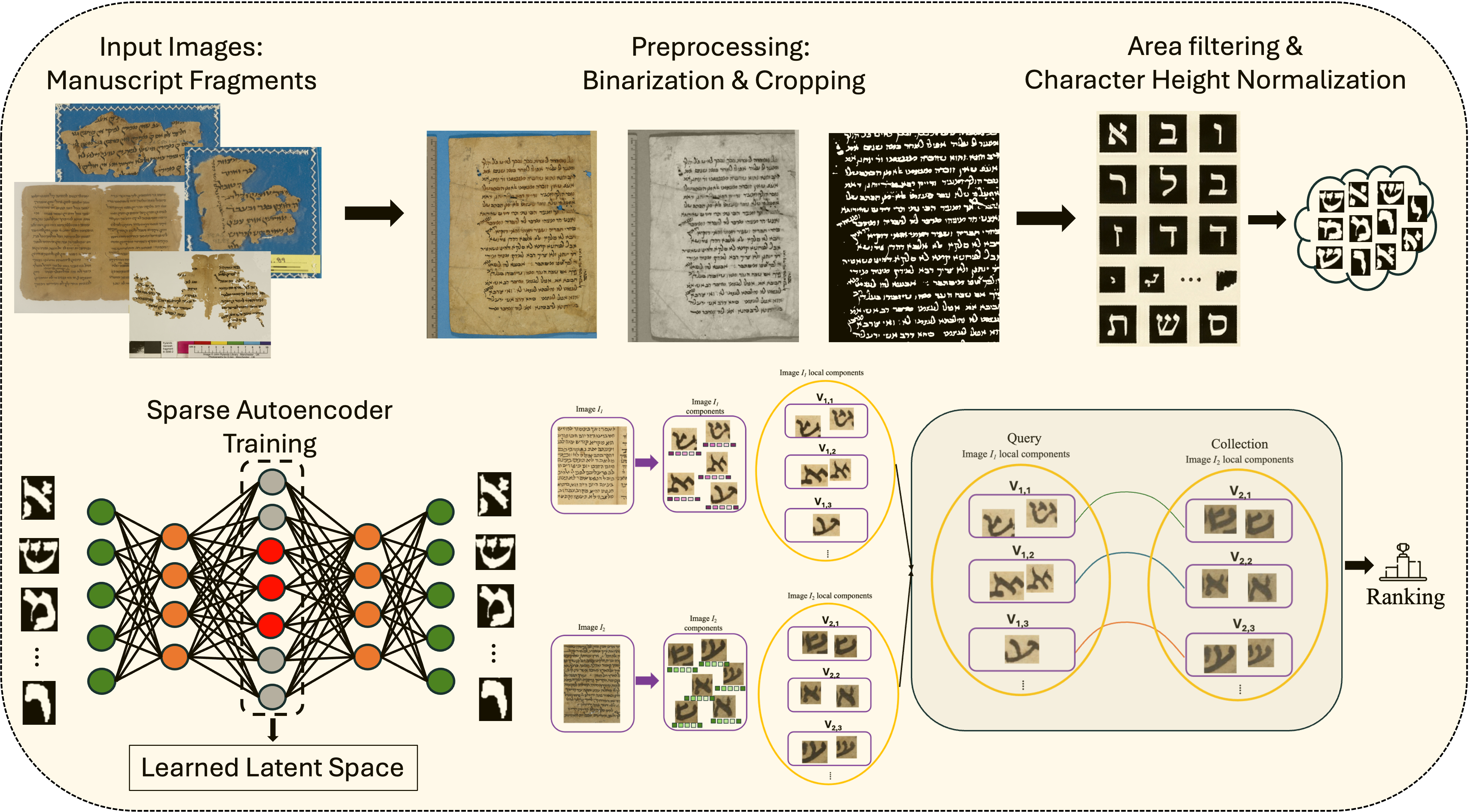}
\caption{Overview of the proposed pipeline.}
\label{fig:pipeline}
\end{figure*}

\subsection{Sparse Convolutional Autoencoder}
\label{subsec:autoencoder}

We encode normalized character patches into a dense feature space using a sparse convolutional autoencoder,
$f_\theta : \mathbb{R}^{64 \times 64} \rightarrow \mathbb{R}^{128}$.
The encoder $f_\theta^{\mathrm{enc}}$ consists of three stride-2 convolutional layers followed by a linear projection. The decoder $f_\theta^{\mathrm{dec}}$ mirrors the encoder using a linear expansion followed by three transposed convolutional layers to reconstruct the input patch.
The full model has $\mathop\approx1.1$M parameters at $d=128$.
The network is trained to minimize the mean-squared reconstruction error with an $L_1$ sparsity penalty on the latent representation:

\begin{equation*}
  \mathcal{L}
  = \underbrace{
      \frac{1}{|\mathcal{P}|}\sum_{p}
      \bigl\|f_\theta^{\mathrm{dec}}(f_\theta(p)) - p\bigr\|_2^2
    }_{\text{reconstruction}}
  + \;\lambda\,
    \underbrace{
      \frac{1}{|\mathcal{P}|}\sum_{p}
      \|f_\theta(p)\|_1
    }_{\text{sparsity}},
  \label{eq:loss}
\end{equation*}
where $\lambda = 10^{-5}$ balances reconstruction fidelity and feature sparsity.
For a fragment $I_i$, this produces a set of continuous embeddings
$E_i = \{f_\theta(p) : p \in \mathcal{P}_i\} \subseteq \mathbb{R}^{128}$.

The autoencoder is trained on patches drawn from manuscript images in the dataset. 
Patches are extracted identically for both training and inference using an aspect-ratio-preserving normalization (\cref{sec:stage1}). 
Specifically, each connected component $c \in \widetilde{\mathcal{C}}_i$ is cropped to its bounding box, resized to a canonical extent while preserving aspect ratio, and centered on a zero-padded $64\times64$ image patch. This ensures that the encoder learns representations directly from the isolated character patches used during retrieval and avoids any mismatch between training and deployment.

\paragraph{Training details.}
For each training image, we sample up to 300 valid component patches and discard patches with insufficient foreground content. The model is trained from scratch for 50 epochs using the Adam optimizer with a learning rate of $10^{-3}$ and a batch size of 256, with early stopping based on the reconstruction loss.

\subsection{The Bag-of-Bags (BoB) Representation}
Unlike BoW, which quantizes $E_i$ against a global codebook, BoB generates a page-adaptive representation. 
We partition $E_I$ into $K=20$ clusters using $K$-means in the main
configuration used for \cref{tab:retrieval} and the runtime analysis; \cref{subsec:ablations} studies sensitivity to $K$.
Let $G_1, \dots, G_K$ denote these clusters, and let $n_I = |E_I|$ be the number of component embeddings for page $I$. The page is represented by a set of prototypes and their corresponding masses:
$
  \mu_a = \frac{1}{|G_a|} \sum_{z \in G_a} z,$
  $\pi_a = {|G_a|}/{n_I}$.
The final BoB representation is the ``bag'' of local prototypes:
$B(I) = \{(\mu_a, \pi_a)\}_{a=1}^K$.

\subsection{Distance Formulation}
\label{sec:distances}

We write $V(I)=\{\mu_a\}_{a=1}^K$ for the support set of centroids in the
weighted BoB representation $B(I)=\{(\mu_a,\pi_a)\}_{a=1}^K$.

Given two manuscript pages $I$ and $J$ with local vocabularies,
$V(I) = \{\mu_a\}_{a=1}^{K}$ and 
  $V(J) = \{\nu_b\}_{b=1}^{K}$,
and normalized cluster-population weights $\boldsymbol{\pi}, \boldsymbol{\rho} \in \Delta^K$, we define three set-to-set distances over their page-adaptive prototypes. 

\paragraph{BoB-Chamfer.}
Each prototype independently finds its nearest match in the opposing vocabulary:
\begin{align*}
 & d_{\mathrm{Chamfer}}(I, J) =\\\nonumber
& \frac{1}{2K}
  \left(
    \sum_{a=1}^{K} \min_{b}\,\|\mu_a - \nu_b\|_2
    \;+\;
    \sum_{b=1}^{K} \min_{a}\,\|\nu_b - \mu_a\|_2
  \right).
  \label{eq:chamfer}
\end{align*}
Unlike bipartite assignment, Chamfer imposes no global bijection between prototypes: each centroid independently contributes its nearest-neighbor cost. This rewards partial style overlap, since a fragment may share only a subset of character types with its join candidate while still receiving a low distance, without penalty for unmatched prototypes. This directly reflects the physical structure of damaged Genizah fragments, where material loss or truncation guarantees that only a subset of character styles will be mutually present. Chamfer is $\mathcal{O}(K^2)$ per pair and achieves the strongest empirical retrieval performance across all metrics (\cref{sec:results}). We note that while symmetric, it does not satisfy the triangle inequality in general.

\paragraph{BoB-Hungarian.}
As a formal geometric alternative, we enforce a strict one-to-one bipartite assignment:
\begin{equation*}
  d_{\mathrm{Hung}}(I, J)
  =
  \frac{1}{K}
  \min_{\sigma \in S_K}
  \sum_{a=1}^{K}
  \|\mu_a - \nu_{\sigma(a)}\|_2,
  \label{eq:hungarian}
\end{equation*}
where $S_K$ denotes the permutation group on $K$ elements. Solved via the Hungarian algorithm on the cost matrix $C[a,b] = \|\mu_a - \nu_b\|_2$, this is exactly the discrete Wasserstein-1 distance $W_1(\mathcal{P}_I, \mathcal{P}_J)$ between uniform empirical distributions $\mathcal{P}_I = \frac{1}{K}\sum_a \delta_{\mu_a}$. Because $W_1$ with an $\ell_2$ ground metric is a valid metric, $d_{\mathrm{Hung}}$ satisfies symmetry, identity of indiscernibles, and the triangle inequality---properties Chamfer lacks. Its per-pair cost is $\mathcal{O}(K^3)$.

\paragraph{BoB-OT (mass-weighted).}
We further incorporate true scribal frequency by replacing uniform weights with normalized cluster populations:
\begin{equation*}
  d_{\mathrm{OT}}(I, J)
  =
  \min_{T \in \Pi(\boldsymbol{\pi},\boldsymbol{\rho})}
  \sum_{a=1}^{K}\sum_{b=1}^{K}
  T_{ab}\,\|\mu_a - \nu_b\|_2,
  \label{eq:ot}
\end{equation*}
where $\Pi(\boldsymbol{\pi}, \boldsymbol{\rho})$ is the set of non-negative transport plans with marginals $\boldsymbol{\pi}$ and $\boldsymbol{\rho}$, solved via the POT library~\cite{pot}. 
A prototype accounting for 40\% of a page's components contributes proportionally more to the distance than one accounting for 2\%. BoB-OT improves over $d_{\mathrm{Hung}}$ (+5.6 pp Hit@1), confirming that prototype mass is informative, though both assignment-based distances are ultimately outperformed by Chamfer in this highly-degraded domain.

\paragraph{Theoretical guarantee for BoB-OT.}
While Chamfer excels empirically due to the partial-overlap
structure of manuscript fragments, BoB-OT admits a formal
approximation guarantee that grounds the use of prototype-level
matching and motivates the vocabulary-size ablation.

\def\theproposition{\!}
\begin{proposition}\label{prop:quantization}
For any two pages, $I$ and $J$, let
\[
\begin{aligned}
\mathcal{P}_I &= \frac{1}{n_I}\sum\nolimits_{i=1}^{n_I}\delta_{z_i},
\qquad
\widetilde{\mathcal{P}}_I = \sum\nolimits_{a=1}^{K}\pi_a\,\delta_{\mu_a},\\
\varepsilon_I &= \frac{1}{n_I}\sum\nolimits_{i=1}^{n_I}\|z_i-\mu_{a(i)}\|_2,
\end{aligned}
\]
and define $\mathcal{P}_J$, $\widetilde{\mathcal{P}}_J$, and $\varepsilon_J$
analogously for page $J$. Then
$
\bigl|W_1(\mathcal{P}_I,\mathcal{P}_J) - W_1(\widetilde{\mathcal{P}}_I,\widetilde{\mathcal{P}}_J)\bigr|
\;\le\; \varepsilon_I + \varepsilon_J.
$
\end{proposition}

\begin{proof}
Assigning each $z_i$ to its centroid $\mu_{a(i)}$ defines a feasible transport
plan from $\mathcal{P}_I$ to $\widetilde{\mathcal{P}}_I$, so
$W_1(\mathcal{P}_I,\widetilde{\mathcal{P}}_I) \leq \varepsilon_I$.
The triangle inequality for $W_1$  gives the desired bound.
\end{proof}

This shows that BoB-OT approximates full component-level optimal
transport with error controlled by the within-page $K$-means quantization quality.
Larger $K$ reduces $\varepsilon_I$ and $\varepsilon_J$, tightening this approximation---as confirmed
empirically (\cref{tab:k_ablation}).
BoB-Chamfer achieves stronger retrieval performance due to the
partial-overlap structure of damaged fragments, while BoB-OT provides a
principled mass-aware alternative with formal approximation guarantees.

\subsection{Bag-of-Words (BoW) Baseline}
\label{sec:bow}
We evaluate two BoW variants that share the same encoder as BoB 
but differ in how the global codebook is constructed.
\textbf{BoW-centroids} pools the page-level local visual words 
(BoB prototypes, weighted by cluster population) into a global codebook.
\textbf{BoW-RawPatches} constructs the global codebook directly 
from all raw component embeddings across the dataset, bypassing 
per-page clustering entirely---this is the standard Bag-of-Visual-Words 
formulation and serves as the primary BoW comparator for BoB.
For each page $I$, let $\{(\mu_a,m_a)\}_{a=1}^{K}$ denote its local prototypes and their cluster populations, where $m_a$ is the number of component embeddings assigned to prototype $\mu_a$ and $n_I=\sum_{a=1}^{K} m_a$. 
We fit a global $k$-means codebook,
$C=\{c_r\}_{r=1}^{K_g}$ ($K_g=100$),
on the pooled set of local prototypes, weighted by their populations $m_a$ (equivalently, by repeating each prototype according to its cluster size).

Each page is then represented by a tf-idf weighted histogram $h_I \in \mathbb{R}^{K_g}$. Its term-frequency component is
\begin{equation*}
\mathrm{tf}_I[r]
=
\frac{1}{n_I}
\sum\nolimits_{a=1}^{K}
m_a\,\mathbb{1}(\operatorname{nn}(\mu_a)=c_r),
\label{eq:tf}
\end{equation*}
where  $\operatorname{nn}(\mu_a)$ denotes the nearest global codeword to $\mu_a$, and the indicator function $\mathbb{1}(\cdot)$ equals 1 if its argument holds and 0 otherwise.
The inverse-document-frequency term is
\begin{equation*}
\mathrm{idf}[r]
=
\log\!\left(\frac{N+1}{|\{I:\mathrm{tf}_I[r]>0\}|+1}\right)+1,
\label{eq:idf}
\end{equation*}
where $N$ is the number of pages in the corpus. 
We use a smoothed inverse-document-frequency term, which avoids zero weights for codewords appearing in all pages and remains well-defined if a codeword has zero document frequency. The final representation is obtained by tf-idf weighting followed by $\ell_2$ normalization.

We compare pages via distances between their final page-level BoW representations: Euclidean, cosine, $\chi^2$, and Hellinger.
For Hellinger distance, histograms are first renormalized to unit $\ell_1$ mass. This baseline provides a meaningful shared-vocabulary comparator while discarding the page-specific prototype geometry retained by BoB.

\subsection{Retrieval, Reranking and Complexity} \label{subsec:retrieval} 
For each query page, all other pages are ranked in ascending order of distance, excluding the query page itself. A page is considered relevant if it belongs to the same join cluster as the query. We report Hit@$k$, mAP@$k$, MRR, and Macro-F1@1. Query pages with no positive mate in the dataset are excluded from metric aggregation. 

\paragraph{Computational cost.} BoW comparison costs $\mathcal{O}(K_g)$ per pair once histograms are built ($K_g=100$ operations). BoB-Chamfer costs $\mathcal{O}(K^2)$ (400 operations at $K=20$) and BoB-Hungarian/OT cost $\mathcal{O}(K^3)$ (8000 operations)---a $4\times$ to $80\times$ overhead per pair relative to BoW. This overhead is the price of expressiveness: without per-image quantization, directly matching all $n_I \geq 200$ raw component embeddings via OT would cost $\mathcal{O}(n^3) \approx 10^6$--$10^9$ operations per pair, making gallery-scale retrieval infeasible. At our operating point ($K=20$), the full precomputed distance matrix builds in under five seconds; precomputed lookup at query time is effectively free for all methods. 

\paragraph{Two-stage retrieval.} For larger manuscript collections, where computing or storing a full $N{\times}N$ BoB-OT matrix becomes expensive---BoB-OT costs $\mathcal{O}(K^3)$ per pair versus $\mathcal{O}(K_g)$ for BoW---we propose a two-stage pipeline: BoW-Cosine retrieves the top-$M$ candidates efficiently, and BoB-OT reranks only that shortlist. We evaluate with $M{=}30$, reducing online cost to $\mathcal{O}(M{\cdot}K^3)$ per query independent of gallery size. 
See \cref{fig:pipeline}.
\Cref{tab:retrieval} shows that the two-stage pipeline matches or improves exhaustive BoB-OT across all reported ranking metrics on this benchmark.

\section{Results}
\label{sec:results}
We evaluated retrieval strategies on our benchmark dataset.

\subsection{Retrieval Performance}
\label{subsec:retrieval_results}

\Cref{tab:retrieval} reports the full retrieval comparison across three
families of baselines. \textbf{Pooling baselines} (MeanPool, MaxPool) aggregate
all component embeddings into a single page vector without any clustering;
MaxPool performs poorly (Hit@1 $\leq$ 0.31), confirming that the element-wise
maximum is not an informative aggregation for manuscript embeddings.
MeanPool-Cosine is competitive at 0.545 Hit@1, but substantially below BoB,
showing that a simple global average of component embeddings cannot substitute
for structured page-level vocabulary matching.
\textbf{BoW-centroids} constructs a global codebook over page-level BoB
prototypes; its best variant (L2) achieves 0.545 Hit@1.
\textbf{BoW-RawPatches} is the standard Bag-of-Visual-Words formulation applied
directly to raw component embeddings without any per-page clustering stage;
its strongest variant ($\chi^2$) reaches 0.739 Hit@1 and 0.800 MRR, making it
the most competitive baseline.

The main result is consistent across all ranking metrics: page-adaptive
BoB representations outperform every baseline. BoB-Chamfer achieves the highest
overall performance at 0.784 Hit@1 and 0.841 MRR, improving over the strongest
BoW baseline (BoW-RawPatches-$\chi^2$, 0.739 Hit@1) by +4.5\% absolute (+6.1\% relative), and over the pooling baselines by a larger margin.
Crucially, BoW-centroids is substantially weaker than BoW-RawPatches despite
sharing the same encoder, showing that the gain of BoB does not come from the
encoder alone, but from the combination of page-adaptive clustering and
set-level prototype matching.

Among adaptive distances, BoB-Chamfer outperforms both BoB-Hungarian-L2
and BoB-OT by 0.079 and 0.023 at Hit@1 respectively.
This suggests that in the fragment setting, soft nearest-neighbor matching
is more robust than strict one-to-one assignment: Paired fragments often share
only a subset of local patterns due to truncation and damage, and Chamfer's partial-overlap formulation rewards shared style without penalizing unmatched prototypes.

\begin{table*}[t]
\centering
\caption{%
  Full retrieval comparison on the Genizah join-retrieval benchmark.
  Methods are grouped into pooling baselines, shared-vocabulary BoW baselines
  (BoW-RawPatches: global codebook over raw component embeddings;
   BoW-centroids: global codebook over page-level BoB prototypes),
  and page-adaptive BoB distances.
  Higher is better for all metrics.
  \textbf{Boldface} indicates the best value per column.
}
\label{tab:retrieval}
\resizebox{0.95\textwidth}{!}{
\begin{tabular}{llcccccccc}
\toprule
\textbf{Family} & \textbf{Method} & \textbf{Hit@1} & \textbf{mAP@1} &
\textbf{Hit@5} & \textbf{mAP@5} & \textbf{Hit@10} & \textbf{mAP@10} &
\textbf{MRR} & \textbf{MacroF1@1} \\
\midrule
\multirow{2}{*}{Pooling}
  & MeanPool-Cosine   & 0.545 & 0.545 & 0.773 & 0.469 & 0.841 & 0.475 & 0.642 & 0.384 \\
  & MaxPool-L2        & 0.307 & 0.307 & 0.534 & 0.288 & 0.648 & 0.300 & 0.424 & 0.187 \\
\midrule
\multirow{4}{*}{BoW-centroids}
  & L2                & 0.545 & 0.545 & 0.682 & 0.506 & 0.784 & 0.516 & 0.619 & 0.397 \\
  & Cosine            & 0.534 & 0.534 & 0.739 & 0.509 & 0.841 & 0.520 & 0.617 & 0.354 \\
  & $\chi^2$           & 0.511 & 0.511 & 0.716 & 0.510 & 0.830 & 0.525 & 0.622 & 0.352 \\
  & Hellinger         & 0.500 & 0.500 & 0.716 & 0.505 & 0.830 & 0.519 & 0.616 & 0.344 \\
\midrule
\multirow{4}{*}{BoW-RawPatches}
  & L2                & 0.659 & 0.659 & 0.795 & 0.562 & 0.852 & 0.564 & 0.730 & 0.543 \\
  & Cosine            & 0.670 & 0.670 & 0.841 & 0.605 & 0.898 & 0.611 & 0.752 & 0.516 \\
  & $\chi^2$           & 0.739 & 0.739 & 0.852 & 0.649 & 0.898 & 0.655 & 0.800 & 0.584 \\
  & Hellinger         & 0.705 & 0.705 & 0.841 & 0.631 & 0.909 & 0.640 & 0.774 & 0.560 \\
\midrule
\multirow{3}{*}{BoB (ours)}
  & Hungarian         & 0.705 & 0.705 & 0.864 & 0.665 & 0.909 & 0.660 & 0.771 & 0.580 \\
  & OT (mass-weighted)& 0.761 & 0.761 & 0.875 & 0.671 & 0.886 & 0.668 & 0.814 & 0.615 \\
  & \textbf{Chamfer}  & \textbf{0.784} & \textbf{0.784} & \textbf{0.898} &
    \textbf{0.734} & \textbf{0.932} & \textbf{0.736} & \textbf{0.841} & \textbf{0.675} \\
\midrule
Two-stage & BoW-Cosine $\rightarrow$ \\
& BoB-OT ($M$=30)
  & 0.771 & 0.771 & 0.875 & 0.691 & 0.901 & 0.668 & 0.814 & 0.615 \\
\bottomrule
\end{tabular}
}
\end{table*}

\subsection{Distance-Space Separation}
\label{subsec:separation}
To probe whether fragment adaptive vocabularies induce a more separable retrieval geometry, we compare the distributions of intra-cluster and inter-cluster distances.
\Cref{tab:separation} reports this analysis for BoW-Cosine and BoB-Hungarian-L2. BoB-Hungarian-L2 shows a larger intra/inter-cluster mean gap (0.343 vs.\@  0.270), slightly stronger KS separation (0.665 vs.\ 0.641), and nearly identical AUC separation (0.902 vs.\@  0.899). We use Hungarian-L2 for this analysis because, as a valid metric (\cref{sec:distances}), it provides a geometrically principled distance space for comparing intra- and inter-cluster distance distributions. We view this as supporting diagnostic evidence for the broader BoB design, rather than as a complete explanation of the strongest retrieval results, which are achieved by BoB-Chamfer.

\begin{table}[t]
\centering
\caption{Distance-separation statistics for intra-cluster vs.\@ inter-cluster
pairs. Larger mean gap, KS statistic, and AUC separation indicate stronger
distributional separation.}
\label{tab:separation}
\setlength{\tabcolsep}{2pt}
\resizebox{\columnwidth}{!}{
\begin{tabular}{lcccccc}
\toprule
\textbf{Method} & \textbf{Intra} & \textbf{Inter} & \textbf{Gap}
  & \textbf{KS} & \textbf{AUC} & \textbf{Cohen's $d$} \\
\midrule
BoW-Cosine        & 0.347 & 0.617 & 0.270 & 0.641 & 0.899 & 1.897 \\
BoB-Hungarian-L2  & 1.735 & 2.078 & 0.343 & 0.665 & 0.902 & 1.831 \\
\bottomrule
\end{tabular}
}
\end{table}

\subsection{Ablation Studies}
\label{subsec:ablations}

We examine sensitivity of retrieval performance to the main design choices: vocabulary size $K$, latent dimension $d$, sparsity regularization, and component normalization. All ablations use BoB-Hungarian-L2 unless stated otherwise.

\paragraph{Local vocabulary size $K$.}
\Cref{tab:k_ablation} varies the number of page-adaptive visual
words for $K=8\mathop{..}64$ at fixed $d=128$, sparsity enabled.
Performance improves until $K=32$, then degrades at $K=64$.
On this benchmark subset, $K=32$ yields the strongest retrieval metrics. The main results and runtime analysis use $K=20$; we treat
\cref{tab:k_ablation} as a sensitivity study rather than as
evidence of a single universally optimal operating point, since the
best $K$ depends on script complexity and collection scale (see Section~\ref{sec:diss}). 

\begin{table}[t]
\centering
\caption{Ablation on local vocabulary size $K$ for BoB-Hungarian-L2 ($d=128$, sparsity enabled). (MF1\,=\,MacroF1)}
\label{tab:k_ablation}
\setlength{\tabcolsep}{2.5pt}
\begin{tabular}{r|cccccc}
\toprule
\textbf{$K$} & \textbf{Hit@1} & \textbf{mAP@1} & \textbf{Hit@5} &
  \textbf{mAP@5} & \textbf{MRR} & \textbf{MF1@1} \\
\midrule
8  & 0.648 & 0.648 & 0.818 & 0.600 & 0.739 & 0.541 \\
16 & 0.705 & 0.705 & 0.841 & 0.650 & 0.767 & 0.569 \\
20 & 0.705 & 0.705 & 0.864 & 0.665 & 0.771 & 0.580 \\
\textbf{32} & \textbf{0.773} & \textbf{0.773} & \textbf{0.864} &
  \textbf{0.689} & \textbf{0.813} & \textbf{0.647} \\
64 & 0.705 & 0.705 & 0.830 & 0.658 & 0.770 & 0.599 \\
\bottomrule
\end{tabular}
\end{table}

\paragraph{Latent dimension $d$.}
\Cref{tab:d_ablation} varies the encoder dimension while retraining from
scratch. Performance peaks at $d=128$; increasing to $256$ does not improve
and slightly degrades retrieval, suggesting that 128-dimensional embeddings
already capture the discriminative structure of the character embedding space.
We use $d=128$ as the default.

\begin{table}[t]
\centering
\caption{Ablation on latent dimension $d$ (K=32, sparsity enabled).}
\label{tab:d_ablation}
\setlength{\tabcolsep}{2.5pt}
\begin{tabular}{r|cccccc}
\toprule
\textbf{$d$} & \textbf{Hit@1} & \textbf{mAP@1} & \textbf{Hit@5} &
  \textbf{mAP@5} & \textbf{MRR} & \textbf{MF1@1} \\
\midrule
64  & 0.705 & 0.705 & 0.807 & 0.642 & 0.756 & 0.578 \\
\textbf{128} & \textbf{0.773} & \textbf{0.773} & \textbf{0.864} &
  \textbf{0.689} & \textbf{0.813} & \textbf{0.647} \\
256 & 0.648 & 0.648 & 0.807 & 0.604 & 0.719 & 0.546 \\
\bottomrule
\end{tabular}
\end{table}

\paragraph{Activation sparsity regularization.}
\Cref{tab:sparsity_ablation} compares training with and without the
$\ell_1$ latent penalty. 
Enabling sparsity improves Hit@1 ($+0.012$), mAP@1 ($+0.012$), Hit@5
($+0.023$), mAP@5 ($+0.017$), and MRR ($+0.008$). MacroF1@1 also improves substantially under sparsity ($+0.052$). Overall, the sparse model is superior
across all reported metrics; we retain the $\ell_1$ penalty in the default
configuration.

\begin{table}[t]
\centering
\caption{Ablation on $\ell_1$ sparsity regularization for
BoB-Hungarian-L2 ($K=32$, $d=128$).}
\label{tab:sparsity_ablation}
\setlength{\tabcolsep}{1pt}
\begin{tabular}{l|cccccc}
\toprule
\textbf{Sparsity} & \textbf{Hit@1} & \textbf{mAP@1} & \textbf{Hit@5} &
  \textbf{mAP@5} & \textbf{MRR} & \textbf{MF1@1} \\
\midrule
\textbf{On}  & \textbf{0.773} & \textbf{0.773} & \textbf{0.864} &
  \textbf{0.689} & \textbf{0.813} & \textbf{0.647} \\
Off & 0.761 & 0.761 & 0.841 & 0.672 & 0.805 & 0.595 \\
\bottomrule
\end{tabular}
\end{table}

\paragraph{Aspect-ratio-preserving normalization.}
\Cref{tab:fitpad_ablation} compares aspect-ratio-preserving patch
normalization against direct isotropic resizing to $64\times64$.
Preserving aspect ratio improves Hit@1, Hit@5, mAP@5, and MRR.
Direct resizing stretches characters to fill the patch uniformly, distorting
width-to-height ratios that are discriminative for Hebrew letterforms.
We adopt aspect-ratio-preserving normalization as the default preprocessing strategy for both training and retrieval.

\begin{table}[t]
\centering
\caption{Ablation on component patch normalization for BoB-Hungarian-L2.
\emph{Preserved}: aspect-ratio-preserving scale-and-pad.
\emph{Stretched}: direct isotropic resize to $64\times64$.}
\label{tab:fitpad_ablation}
\setlength{\tabcolsep}{1pt}
\begin{tabular}{l|cccccc}
\toprule
\textbf{Method} & \textbf{Hit@1} & \textbf{mAP@1} & \textbf{Hit@5} &
  \textbf{mAP@5} & \textbf{MRR} & \textbf{MF1@1} \\
\midrule
\textbf{Preserved} & \textbf{0.739} & \textbf{0.739} & \textbf{0.841} &
  \textbf{0.671} & \textbf{0.789} & 0.592 \\
Stretched          & 0.716 & 0.716 & 0.830 & 0.652 & 0.778 & \textbf{0.603} \\
\bottomrule
\end{tabular}
\end{table}

\begin{exclude}
\subsection{Computational Trade-Offs}
\label{subsec:complexity_results}

BoW reduces each page to a fixed-length histogram and is therefore cheaper to compare at query time. In contrast, BoB requires set-to-set matching between page-specific prototype vocabularies, making Chamfer, Hungarian matching, and OT progressively more expensive. The purpose of per-page quantization is not to match BoW in cost, but to make prototype-level matching practical relative to raw component-level comparison.

This trade-off motivates the two-stage retrieval setting in \Cref{tab:retrieval}: BoW provides efficient candidate generation, while BoB-OT is applied only to a short list of candidates. At the scale of the current benchmark, full pairwise distance precomputation is inexpensive; for larger galleries, reranking provides a more realistic deployment mode than exhaustive online OT.
\end{exclude}

\subsection{Computational Trade-Offs}
\label{subsec:complexity_results}

\Cref{tab:runtime} reports query-time cost. Precomputed matrix lookup
is effectively free ($\mathop<0.01$\,ms/query) for all methods.
On-the-fly BoB-Hungarian costs $\mathop\approx 10$\,ms/query due to the
$\mathcal{O}(K^3)$ assignment, making exhaustive online matching impractical
for large galleries. The two-stage pipeline (BoW-Cosine
$\rightarrow$ BoB-OT rerank, $M$=30) costs $\mathop\approx 11$\,ms/query but
scales as $\mathcal{O}(M{\cdot}K^3)$ \emph{independent of gallery size},
providing a practical deployment mode for large collections without
requiring a fully precomputed BoB distance matrix.

\begin{table}[t]
\centering
\caption{Query-time runtime profiling. Precomputed lookup is free for all
methods; on-the-fly assignment costs scale with $K$.}
\label{tab:runtime}
\setlength{\tabcolsep}{2pt}
\begin{tabular}{lcc}
\toprule
\textbf{Method} & \textbf{ms/query} & \textbf{Scales w/ $N$} \\
\midrule
BoW-Cosine (precomputed)              & 0.006  & No  \\
BoB-Hungarian (precomputed)           & 0.006  & No  \\
BoB-Hungarian (on-the-fly)            & 10.30  & Yes \\
BoW-Cosine $\rightarrow$ BoB-OT ($M$=30) & 11.33 & No  \\
\bottomrule
\end{tabular}
\end{table}

\section{Discussion}
\label{sec:diss}\label{sec:conclusion}

We have presented Bag-of-Bags (BoB), a manuscript join retrieval framework that replaces a shared global codebook with page-adaptive vocabularies built from sparse autoencoder embeddings of connected-component patches. Pages are compared through set-based distances over these local prototypes, enabling retrieval that is sensitive to page-specific visual structure rather than global frequency statistics alone. We further showed that the mass-weighted BoB-OT variant admits a formal approximation guarantee: its deviation from full component-level optimal transport is bounded by the sum of the two pages' $k$-means quantization errors. On our Genizah benchmark, the best variant, BoB‑Chamfer, achieves a Hit@1 of 78.4\%, a relative improvement of 6.1\% over the strongest BoW baseline. Overall, the results show that, for manuscript join retrieval, modeling how a page organizes its local visual modes is more effective than forcing all pages into a single shared vocabulary. More broadly, adaptive set-based representations appear promising for retrieval problems involving partial, noisy, and heterogeneous visual evidence.
These results are consistent with the intuition that page-adaptive vocabularies preserve page-specific local structure that may be blurred by shared-codebook frequency representations.

\paragraph{Chamfer vs.\@ assignment-based distances.}
Among BoB variants, Chamfer outperforms both Hungarian and OT
despite having weaker formal structure. A plausible
explanation is that fragment pairs often share only a subset
of local character patterns: Chamfer's nearest-neighbor
formulation rewards this partial overlap without penalizing
unmatched prototypes, whereas one-to-one assignment forces
every prototype to contribute regardless of whether it has a
meaningful counterpart on the opposite page. This suggests
that for partial-document retrieval, soft set matching may be
more robust than globally constrained correspondence.

\paragraph{Limitations.}
The evaluation is constrained to a manually annotated subset of 287 images forming 100 join clusters, reflecting the difficulty of obtaining verified, high-confidence ground-truth joins from domain experts. 
The complete Genizah corpus contains over 250,000 fragments, and the pipeline is designed with that full-corpus setting in mind. 
We expect the per-image vocabulary construction to remain advantageous at scale: adapting to each fragment's local character distribution rather than a fixed global codebook becomes more valuable as stylistic diversity increases. 
Empirical validation on a larger annotated sample remains an important direction for future work, but awaits the accumulation of ground truth. 
We focus on BoW-family baselines sharing the same sparse encoder so as to isolate the effect of page-adaptive vocabulary construction; comparison against more recent learned retrieval models is another important direction for the future.
Extending the pipeline to smaller fragments and more heavily degraded pages is a natural
direction for future work. 
On our benchmark test set, $K=32$, the approximate number of  graphemes for those texts, yielded the strongest retrieval metrics,
at non-negligible cost compared to the $K=20$ used in most of our experiments.
It remains to be seen what the ideal number is for the complete corpus.
It may in fact be best to form $K'>K$ clusters and retain only the
$K$ most populated prototypes for matching, decoupling vocabulary richness from matching cost.


It would be interesting to investigate to what extent image-adaptive prototype vocabularies are also beneficial in other computer-vision settings that depend on subtle local visual cues, such as fine-grained recognition. In such problems, instances may be globally similar while differing in small local patterns, and matching may rely on partial correspondences rather than full global appearance agreement. This makes them a natural setting in which to study whether set-to-set matching over image-specific prototype vocabularies might provide an advantage over shared global-codebook representations.

\section*{Acknowledgments}
We are grateful to the Friedberg Genizah Project (FGP) and Dr.\@ Roni Shweka for their contribution to the construction of the dataset. 
This research was funded in part by the European Union (ERC, MiDRASH, Project No.\@ 101071829. Principal investigators: Nachum Dershowitz, Tel Aviv University; Judith Olszowy-Schlanger, EPHE-PSL;
Avi Shmidman, Bar-Ilan University;
and Daniel Stökl Ben Ezra, EPHE-PSL).
Views and opinions expressed are, however, those of the authors only and do not necessarily reflect those of the European Union or the European Research Council Executive Agency. Neither the European Union nor the granting authority can be held responsible for them.

{
    \small
    \bibliographystyle{ieeenat_fullname}
    \bibliography{main}
}


\end{document}